\documentclass{article}
\PassOptionsToPackage{dvipsnames}{xcolor}
\usepackage{ijcai25}

\usepackage{times}
\usepackage{soul}
\usepackage{url}
\usepackage[hidelinks]{hyperref}
\usepackage[utf8]{inputenc}
\usepackage[small]{caption}
\usepackage{graphicx}
\usepackage{amsmath}
\usepackage{amsthm}
\usepackage{booktabs}
\usepackage{algorithm}
\usepackage[noend]{algorithmic}
\usepackage[switch]{lineno}

\usepackage{amsfonts}
\usepackage{tikz}

\usepackage{multicol,multirow}
\usepackage{lipsum}
\usepackage{amsfonts}
\usepackage{amsmath}
\usepackage{amssymb}

\DeclareSymbolFont{matha}{OML}{txmi}{m}{it}
\DeclareMathSymbol{\varv}{\mathord}{matha}{118}

\newtheorem{theorem}{Theorem}
\newtheorem{definition}{Definition}

\newtheorem{proposition}{Proposition}
\newtheorem{lemma}{Lemma}

\newcommand{\checkit}[1]{}

\newcommand{\tuple}[1]{\langle #1 \rangle}

 \renewcommand{\P}{\mathcal{P}}
 
 \newcommand{\T}{\mathcal{T}}
\newcommand{\U}{\mathcal{U}}

\newcommand{\prop}[1]{{\cal P}_{A}}

\newcommand{\thereals}{\ensuremath{\mathbb{R}}}

\usepackage{xspace}

\newcommand{\counters}[0]{\textsc{Counters}\xspace}
\newcommand{\blockgrouping}[0]{\textsc{Blocks-grouping}\xspace}
\newcommand{\sailing}[0]{\textsc{Sailing}\xspace}
\newcommand{\drone}[0]{\textsc{Drone}\xspace}
\newcommand{\procurement}[0]{\textsc{Procurement}\xspace}
\newcommand{\cashpoint}[0]{\textsc{Cashpoint}\xspace}
\newcommand{\terraria}[0]{\textsc{Terraria}\xspace}

\title{Handling Infinite Domain Parameters in Planning Through Best-First Search with Delayed Partial Expansions
}

\author{
    Submission 7934
}

\author{
Ángel Aso-Mollar$^1$
\and
Diego Aineto$^2$\and
Enrico Scala$^2$\And
    Eva Onaindia$^1$\\
\affiliations
$^1$Valencian Research Institute for Artificial Intelligence (VRAIN), Universitat Politècnica de València\\
$^2$Università degli Studi di Brescia\\
\emails
\{aaso,onaindia\}@vrain.upv.es,
\{diego.ainetogarcia,enrico.scala\}@unibs.it
}

\begin{document}

\maketitle

\begin{abstract}
In automated planning, control parameters extend standard action representations through the introduction of continuous numeric decision variables. Existing state-of-the-art approaches have primarily handled control parameters as embedded constraints alongside other temporal and numeric restrictions, and thus have implicitly treated them as additional constraints rather than as decision points in the search space. In this paper, we propose an efficient alternative that explicitly handles control parameters as true decision points within a systematic search scheme. We develop a best-first, heuristic search algorithm that operates over infinite decision spaces defined by control parameters and prove a notion of completeness in the limit under certain conditions. Our algorithm leverages the concept of delayed partial expansion, where a state is not fully expanded but instead incrementally expands a subset of its successors. Our results demonstrate that this novel search algorithm is a competitive alternative to existing approaches for solving planning problems involving control parameters.\end{abstract}


\section{Introduction}

In classical planning, the number of instantiated actions is determined by the finite number of objects in the world. The first infinite-domain variables introduced in PDDL (Planning Domain Description Language) were the variable time and numeric variables (fluents) to represent numeric resources in PDDL2.1 \cite{Fox_2003}. The semantics of PDDL2.1, however, explicitly rule out number-valued arguments in actions in order to keep the logical state space finite. Years later, a proposal was made to extend PDDL2.1 to allow actions with continuous numeric parameters called \textit{control parameters}. 

Control parameters are intended to represent physical quantities that a planner might choose to make an action have a desired effect. They were introduced by the POPCORN planner \cite{SavasFLM16,Savas18}, and have been also recently incorporated into the temporal and numeric planner NextFLAP \cite{SAPENA2024122820}. In both planners, control parameters are implicitly handled using Linear Programming and Satisfying Modulo Theory solvers, respectively, by jointly formulating temporal and control parameter constraints. 
Other methods use Neural Networks to concretize an abstract plan, determining the control parameters of the actions and the feasibility of the input-output evaluations via optimization of the input \cite{Heesch2024ALA}. This latter approach circumvents search by handling control parameters as black boxes.

In this work, we aim to explore an approach that treats control parameters as explicit elements in the decision space rather than as constraints to be satisfied. We define a systematic search algorithm for planning problems with control parameters that (1) defines control parameters as numeric decision points; and (2) guarantees a notion of completeness. 
Our algorithm, Sampling Best-First Search (S-BFS), enables systematic search in structured infinite decision spaces. This is achieved through \textit{delayed partial expansions}, where nodes are partially expanded by sampling successors and re-expanding them if they appear promising in future evaluations.

This paper is structured as follows. In Section 2, we present the formalization of a planning problem with control parameters. 
Section 3 provides a detailed description of the S-BFS algorithm, while Section 4 introduces a notion of completeness under specific conditions, along with other relevant properties. Section 5 reviews related work on resolution strategies for problems that involve continuous numeric variables. Finally, Section 6 presents the experimental evaluation and we compare the performance of S-BFS to other approaches that manage numeric control  parameters.


\section{Problem Formalization} \label{background}

We adopt the numeric planning formalization from \cite{subgoaling20} and adapt it to our needs. We consider a set of numeric variables, $X$, and a set of Boolean variables, $F$, referred to as \textit{numeric state variables} and \textit{Boolean state variables}, respectively. We define $V = X \cup F$ as the union of these sets. A \textbf{valuation} $v : V \rightarrow Dom$ maps each variable $\lambda \in V$ to a value in its domain $Dom$, denoted as $v[\lambda]$. A \textbf{numeric expression} is a polynomial expression over $X$, $\text{Expr}(X)$ being the set of all numeric expressions. A \textbf{constraint} involves equalities $f = b$ for variables $f \in F$, where $b \in \mathbb{B}$ ($\mathbb{B}=\{\top,\bot\}$), and inequalities $\xi \bowtie 0$, such that $\bowtie\ \in \{<,\leq, =, \geq, >\}$, where $\xi \in \textup{Expr}(X)$, combined with logical operations $\land, \lor, \lnot$. The set of constraints over $V$ is defined as $\text{Constr}(V)$. \textbf{Propositional assignments}, defined as $\text{Assign}_P(F)$, consist of assignments $f := b$, where $f \in F$ and $b \in \mathbb{B}$, while \textbf{numeric assignments}, $\text{Assign}_N(X_1, X_2)$, are of the form $x := \xi$, where $x \in X_1$ and $\xi \in \text{Expr}(X_2)$.

We propose an alternative yet equivalent formulation to the control parameters approach proposed in \cite{Savas18}, as it offers a more convenient framework for the theoretical developments introduced later in this work. 
Specifically, we define a new set of numeric variables, referred to as \emph{control variables}, $U$, in addition to the existing set of numeric state variables, $X$. Control variables are unconstrained in their evolution and can assume any value within their (bounded) infinite domain, i.e., an interval. Rather than making decisions by directly evaluating a numeric parameter within an action, we evaluate fluents that will subsequently be used in actions. This subtle difference transforms the decision space into pairs of actions and values, but is fundamentally equivalent to the notion of control parameters, whose decision space is defined as actions with grounded numeric values.

In a numeric planning problem with control variables, the values of the state variables $X$ change after the execution of actions, based on the assignments explicitly defined in their effects, which now incorporate control variables: $\text{Assign}_N(X, X \cup U)$. For example, we can model increments $x := x + u$, such that $u$ takes any value on an associated interval, say $u \in [1,2]$. Moreover, control variables also appear in action preconditions. This definition removes the need to define control parameters explicitly for each action.

\begin{definition}
    A \textbf{numeric planning problem with control variables} is a tuple $\mathcal{P} = \langle F, X \cup U, A, I, G\rangle$, where:
    \begin{itemize}
        \item $F$ is a finite set of Boolean state variables;
        \item $X$ is a finite set of numeric state variables over $\mathbb{R}$;
        \item $U$ is a finite set of bounded numeric control variables over $[l,u]$, $l,u \in \mathbb{Z}$;
        \item $A$ is a {finite} set of actions $a = (\textup{Pre}(a),\textup{Eff}(a))$, where $\textup{Pre}(a) \in \textup{Constr}(X \cup F \cup U)$ and $\textup{Eff}(a) \subseteq Assign_P(F) \cup Assign_N(X,X \cup U)$;
        \item $s_0$ is a valuation over $X \cup F$, the initial state;
        \item $G \in \textup{Constr}(X \cup F)$ is the goal condition.
    \end{itemize}
\end{definition}

The value range of the control variables need not be continuous; it can be discretized according to a chosen precision, as discussed in \cite{aso-mollar2025a}.
Next, we characterize the semantics of a numeric planning problem with control variables as a transition system.

\begin{definition} 
    The semantics of a \textbf{numeric planning problem with control variables} $\mathcal{P}$ is the transition system defined as  $\mathcal{T}(\mathcal{P})=\langle S,\U,s_0,S_G, \rightarrow\rangle$ where:
    \begin{itemize}
        \item $S = \langle \mathbb{B}^{|F|} \times \mathbb{R}^{|X|} \rangle$ is the state space; a state $s \in S$ is a valuation over $X \cup F$;
        \item $\U = \langle [l_1,u_1] \times \ldots \times [l_n,u_n]\rangle $ is the control space; $\mu \in \U$ is a valuation over $U$;
        \item $s_0 \in S$ is the initial state described by $I$, $s_0=I[X\cup F]$;
        \item $S_G = \{ s \in S\ |\ s \models G\}$ are the goal states
        \item $\rightarrow$ is the transition relation. A transition $(s,\langle a,\mu\rangle,s') \in S \times A \times \U \times S$ belongs to $\rightarrow$ iff $(s,\mu) \models Pre(a)$ and, $\forall v\in X \cup F,$ 
        \[s'[v] = 
        \begin{cases}
            s[\xi] & if\ (v:= \xi) \in \textup{Eff}(a) \\
         s[v]& otherwise 
        \end{cases}
        \]
    \end{itemize}
\end{definition}

The decision space of a state $s$ is defined as $D(s):=\{\langle a,\mu\rangle \in A \times \U: \exists s'\in S \textup{ such that } (s,\langle a,\mu\rangle,s') \in \rightarrow\}$. This decision space is infinite, since $\mu$ takes values in an infinite set (an interval). A plan in this setting is no longer a sequence of actions, but a sequence of pairs action-valuation of control variables:

\begin{definition}\label{def:input-action-plan}
    Let $\mathcal{P}$ be a numeric planning problem with control variables and let $\mathcal{T}(\mathcal{P})=\langle S, \U, s_0, S_G, \rightarrow\rangle$ be its associated transition system. A \textbf{plan} $\pi = (\tuple{a_i,\mu_i})_{i=1}^n$ for $\mathcal{P}$ is a \textbf{solution} iff $(s_0, \langle a_0,\mu_0\rangle,s_1),\dots,(s_{n-1}, \langle a_n,\mu_n\rangle,s_n) \in\  \rightarrow$ and $s_n \in S_G$. We denote the set of all solutions of $\P$ as $\Pi_\P$. We will say that $\P$ is \textbf{solvable} if $\Pi_\P \neq \emptyset$.
\end{definition}


\section{Sampling Best-First Search}

We approach our problem using Best-First Search (BFS), a widely used strategy for finite decision spaces. BFS explores a decision tree, where states are nodes, and selects nodes for expansion based on a \textit{node evaluation criterion} (NEC), denoted as $f$; selection is managed through an open list, typically a priority queue ordered by $f$. This criterion is usually a heuristic function, $f = h$, or a combination of heuristic and accumulated cost, $f = g + h$. Once selected, a node is \textit{expanded} by generating its successors, linking them to the parent, and inserting them into the open list. BFS is considered a \textit{systematic search} strategy as it explores all choices. However, in the infinite decision space induced by control variables, a node may have infinitely many successors, preventing full expansion. As a result, standard BFS algorithms are not directly applicable to our problem. To overcome this limitation, we propose two modifications to the conventional BFS framework.

Firstly, since we cannot fully expand a node as it has infinitely many successors, we use \textbf{delayed partial expansions} to incrementally generate subsets of a state's successors via a sampling function $\phi$. 
\begin{definition}
    A \textbf{sampling function} $\phi$ for a problem $\P= \langle F, X \cup U, A, s_0, G\rangle$ with transition system $\T(\P) = \langle S,\U,s_0,S_G, \rightarrow\rangle$ is a function $\phi:S \rightarrow \bigcup_{s \in S} \textup{P}_d(D(s))$ that defines a probability density for every $s\in S$, with domain in its decision space $D(s)=\{\langle a,\mu\rangle \in A \times \U: \exists s'\in S \textup{ such that } (s,\langle a,\mu\rangle,s') \in \rightarrow\}$, where $\textup{P}_d$($X$) denotes the set of all probability density functions with domain $X$:
    \[\phi(s): D(s) \rightarrow [0,1]\]
\end{definition}

In other words, $\phi$ assigns a probability density function to every state $s$, and each density function $\phi(s)$ is defined over the decision space of $s$, $D(s)$.

Secondly, partially expanded states cannot be closed, meaning they will be re-added to the open list for potential (partial) expansion in subsequent iterations. However, their NEC value will be adjusted based on a \textbf{rectification function} $r_h$, which is defined based on a heuristic function $h$. 
Such functions are generalizations of heuristic functions that depend on the number of delayed partial expansions. 

\begin{definition}
    Given a problem $\P$ that defines the transition system $\T(\P) = \langle S,\U,s_0,S_G, \rightarrow\rangle$, a \textbf{ rectification function} $r_h$, where $h$ is a heuristic function $h:S \rightarrow \thereals$, is a function $r_h :\mathbb{Z}^+ \times S \rightarrow \mathbb{R}$ such that \[r_h(0,s)=h(s),\ \forall s \in S\]
\end{definition}

Using these two components, we define the \textbf{Sampling Best-First Search} (S-BFS) schema. S-BFS denotes a family of algorithms that integrate delayed partial expansions guided by a sampling function $\phi$ and NEC corrections determined by a rectification function $ r_h $. We will address specific instances of this algorithm as S-BFS$_{\phi,r_h}$. Concretely,  $f=r_h$ wil be referred as S-G$_{\phi,r_h}$, and $f=g+r_h$ as S-A$_{\phi,r_h}$.

Algorithm \ref{alg:sgbfs} outlines S-BFS$_{\phi,r_h}$. It starts by initializing a priority queue $Open$ with the initial state $s_0$ and its $f$-value (line 1). As long as the list is not empty (line 2), the state with the lowest NEC value is selected (line 3). If it is a goal state (line 4), the algorithm returns that the goal has been reached (line 5). If not (line 6), $s$ is partially expanded by sampling from the density function $\phi(s)$ and generating one successor $s'$ (line 7-8). It is inserted to $Open$ (line 9), and the original state $s$ is reinserted with its rectified value (line 10-11).

\begin{algorithm}[tb]
\caption{S-BFS$_{\phi,r_h}$ }
\label{alg:sgbfs}
\textbf{Input}: Transition system $\mathcal{T}(\mathcal{P})=\langle S,\U,s_0,S_G, \rightarrow\rangle$, sampling function $\phi$, rectification function $r_h$ and NEC $f$\\
\textbf{Output}: Goal reached
\begin{algorithmic}[1]
\STATE $Open:=\{(f(s_0),s_0)\}$
\WHILE{$Open \neq \emptyset $}
    \STATE Extract node $(f(s),s)$ with lowest $f$-value

    \IF{$s \in S_G$}
        \RETURN true
    \ELSE
        \STATE Sample $\langle a,\mu\rangle$ from $\phi(s)$
        \STATE Apply $\langle a,\mu \rangle$ to $s$ and generate $s'$
        \STATE Insert $(f(s'),s')$ in $Open$ 
        \STATE Rectify $f(s)$ using $r_h$
        \STATE Insert $(f(s),s)$ in $Open$
    \ENDIF
\ENDWHILE
\RETURN false
\end{algorithmic}
\end{algorithm}

While the algorithm is designed for infinite decision spaces where the probability of re-sampling a successor is zero, it can be seamlessly extended to hybrid scenarios in which actions with control variables and actions without control variables are interleaved, by checking if the sampled node has already been visited.


\section{Properties of S-BFS Algorithms}

In this section, we use a relaxed notion of completeness that works for infinite decision spaces: probabilistic completeness. Then, we outline the essential properties that the sampling function $\phi$ and the rectification function $r_h$ must satisfy to ensure the probabilistic completeness of S-BFS. After that, we demonstrate the probabilistic completeness of S-BFS given those function properties.

\subsection{On the Completeness of Algorithms for Infinite Decision Spaces}

Since the decision space of a numeric planning problem with control variables is infinite, we need to define a specialized notion of completeness for systematic search algorithms that applies to infinite decision spaces. One might question how an algorithm with a finite number of steps can solve problems involving an infinite number of successors. To address this, we introduce the concept of probabilistic completeness, as defined in \cite{Valenzano_Xie_2016}. This definition is more relaxed than the traditional definition of completeness, since it only requires that an algorithm finds a solution in the limit with probability 1:

\begin{definition}
     A search algorithm is \textbf{probabilistically complete} if, for every solvable problem $\P$, i.e., for which the set of all solutions $\Pi_\P=\{\pi\ : \pi \textup{ is a solution for } \P\}$ is not empty, the probability of finding a solution $\pi\in \Pi_\P$ in $n$ steps is 1 when $n \rightarrow \infty$.
 \end{definition}

\subsection{Properties of S-BFS}

First of all, we need to define the support of an arbitrary function:

\begin{definition}
    The \textbf{support} of any function $f: X \rightarrow \mathbb{R}$ is defined as the set of points where $f$ is strictly positive.
    \[supp(f) := \{x \in X :\ f(x)>0\}\]
\end{definition}
    
With this, we can define our distributions of interest. We are interested in sampling functions $\phi$ that define probability densities $\phi(s),\ \forall s \in S$,  such that $\phi(s)$ has \textbf{support within its whole domain} $D(s)$. This is needed because, when sampling a new successor state, any possible successor should have a non-negative probability of being sampled. An example of this could be the uniform sampling function, $\phi_u$, defined such that $\phi(s)$ is a uniform distribution $U(D(s))$ for every $s\in S$.

Next, we must define a subset of rectification functions that possess a crucial property. This property ensures that re-inserted nodes do not dominate the search, thereby preventing infinite loops.

\begin{definition}
    Given a rectification function $r_h :\mathbb{Z}^+ \times S \rightarrow \mathbb{R}$, where $h$ is a heuristic function $h: S \rightarrow \thereals$, $r_h$ is a \textbf{proper} rectification function iff $ \exists n_0 \in \mathbb{Z}^+$ such that, after $n_0$,
    $r_h$ is monotonically increasing with respect to the first variable
    \[
    r_h(n_1,s) < r_h(n_2,s),\ \forall n_1,n_2\ :\ n_0 \leq n_1 < n_2
    \]
\end{definition}

Given a heuristic function $h$, the rectification function $r_h(n,s)=h(s)+n$ is dominated by $h$. This property of rectification functions allows distinct  balance of exploration and exploitation in the algorithm, while ensuring its probabilistic completeness. With these two considerations, we can now prove the probabilistic completeness of S-BFS.

\begin{theorem}
\label{theorem1}
Let $\mathcal{P} = \langle F, X \cup U, A, s_0, G\rangle$ be a numeric planning problem with control variables that induces the transition system $\mathcal{T}(\mathcal{P})=\langle S, \U, s_0, S_G, \rightarrow\rangle$. Let $\phi$ be a sampling function such that $supp(\phi(s))=D(s)\ \forall s\in S$ and let $r_h$ be a proper rectification function, where $h$ is a heuristic function. Then, the probability of S-BFS$_{\phi,r_h}$ finding a solution $\pi\in\Pi_\P$ at step $n$, if $\Pi_\P \neq \emptyset$, is 1 as $n \to \infty$. \end{theorem}

\begin{proof}
There are two aspects to prove for this theorem. The first is that (1) any state in the search space can be sampled and inserted into the priority queue. The second is that (2) the priority queue ensures every state is eventually selected. Once these two conditions have been proven, it can be stated that (3) the S-BFS$_{\phi,r_h}$ algorithm will be able to reach a solution, if any, since every state will eventually be visited.

(1) First, we have to prove that every successor state will eventually be sampled to enter the priority queue. Since $supp(\phi(s))=D(s),\ \forall s \in S$, every successor state has a strictly positive probability of being sampled. Let us consider an arbitrary current state $s\in S$; we define $P_{\phi(s)}(s'|s) = p > 0$ as the probability of sampling a successor state $s'$ (i.e., $(s,\langle a,\mu\rangle, s') \in\ \rightarrow$ for some $\langle a,\mu\rangle \in D(s)$) given state $s$. This probability is strictly positive, as stated before, because $supp(\phi(s)) = D(s)\ \forall s\in S$. Let us define the following probabilistic event:
\[A_k = \{\textup{Given }s,\textup{ the state }s'\textup{ is not sampled after }k\textup{ trials}\}\]
The probability of the event $A_k$ can be determined using $p$:
\[P_{\phi(s)}(A_k)=(1-p)^k\]
Since $p>0$, $|1-p|<1$. We can characterize the infinite sum of $P_{\phi(s)}(A_k)$, which turns to be a geometric series
\[\sum_{k=1}^\infty P_{\phi(s)}(A_k) = \sum_{k=1}^\infty (1-p)^k = \frac{1}{1-(1-p)} = \frac{1}{p} < \infty\]
And if we apply the Borel-Cantelli Lemma: 
\begin{lemma}
    (Borel-Cantelli Lemma). Let $P$ be a probability distribution, and let $\{A_k\}$ be a sequence of probabilistic events such that $\sum_{k=1}^\infty P(A_k)<\infty$, then $P(\limsup_{k\rightarrow \infty}A_k)=0.$
\end{lemma}
We prove that $P_{\phi(s)}(\limsup_{k\rightarrow \infty}A_k)=0$, i.e., the probability of never sampling an arbitrary successor state is zero, so every successor will eventually be sampled, q.e.d.

(2) Second, we have to prove that every state inside the priority queue will eventually be first in priority. Since $r_h$ is a proper rectification function, there exists a $n_0$ that, for all $n \geq n_0$, $r_h$ is monotonically increasing with respect to its first variable. This condition ensures that no state re-inserted into the queue will dominate in priority indefinitely.
Let $s$ be a state that has not yet been expanded, with NEC $f(0,s)=g(s)+r_h(0,s)=g(s)+h(s)$. If it currently has the minimum NEC in the priority list, it will be selected. If not, other finite number of states with smaller NEC values will be selected; for every such state, its NEC value will eventually increase and surpass $g(s)+h(s)$. This process ensures that no state can indefinitely block another state from being selected. Consequently, every state in the priority queue will eventually become the first in priority, q.e.d.

(3) Finally, for a problem $\P$, S-BFS$_{\phi,r_h}$ will start from the initial state and will eventually be able to reach every state in the transition system, as it is ensured that every successor state can be sampled with probability $1$ in the limit, and that every sampled successor can be selected from the priority queue, preventing the algorithm from entering infinite loops. Thus, it will be able to reach every possible transition in $\rightarrow$ and eventually reach a final state $s\in S_G$, if $\Pi_\P\neq \emptyset$.
\end{proof}

The convergence of S-G algorithms can be expected to be faster than that of S-A algorithms. However, S-A algorithms are expected to yield higher-quality solutions since they include information about the current cost. In the following, we characterize a theoretical property of the search in S-BFS algorithms that will eventually let us bound the quality of the solution for S-A algorithms.

\begin{proposition}
\label{prop1}
    Let $\phi$ be a sampling function and $r_h$ a rectification function, both satisfying the conditions stated in Theorem \ref{theorem1}. For a given node $s$ in the current search tree, S-BFS$_{\phi,r_h}$ ensures that the $f$-value of every non-leaf node in the current subtree rooted at $s$ is bounded by $f(s)$.
\end{proposition}
\begin{figure}
    \centering
    \includegraphics[width=\linewidth]{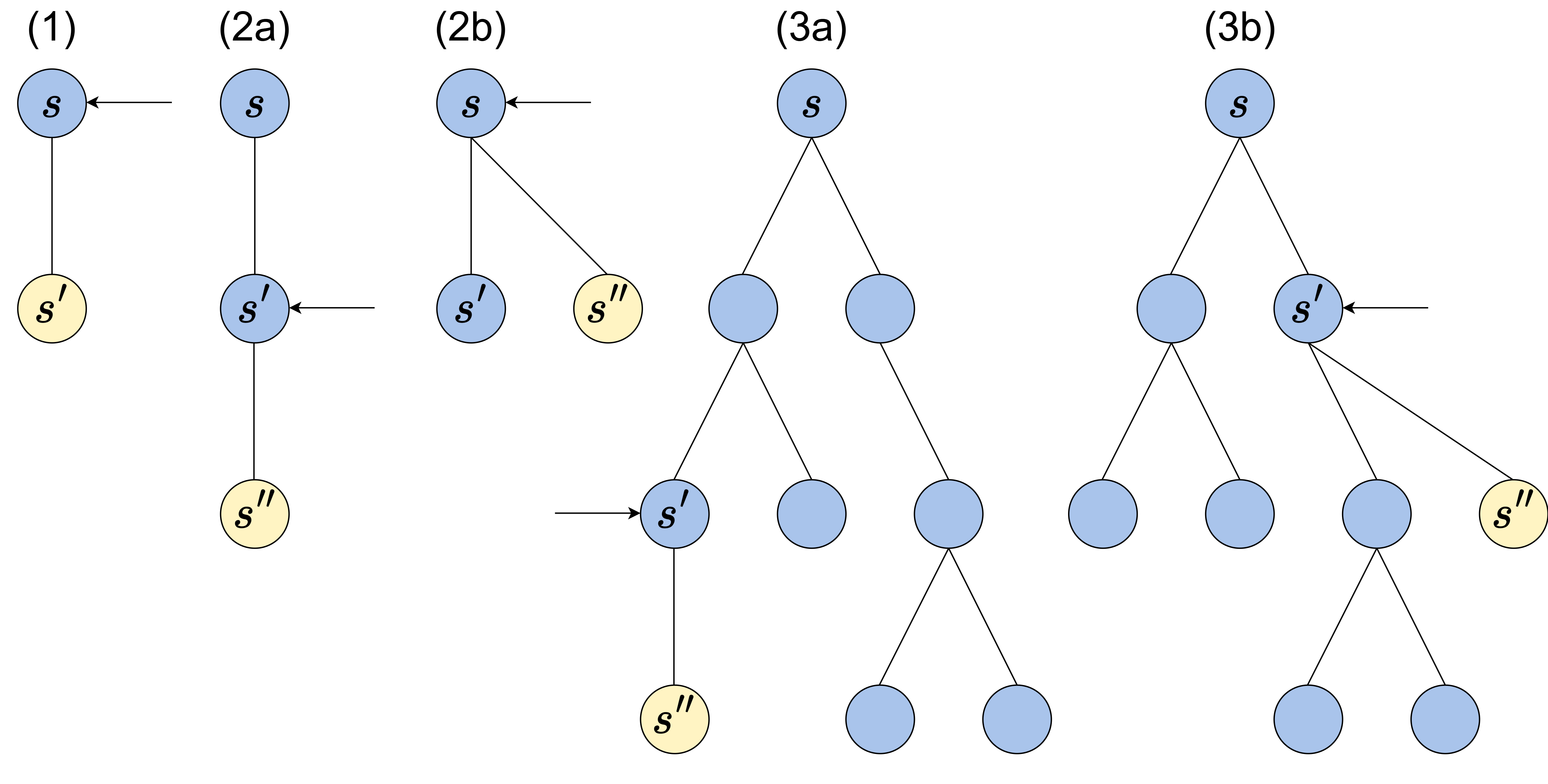}
    \caption{Figure to support the demonstration of Proposition \ref{prop1}. Existing nodes are marked blue, sampled nodes are marked yellow and selected nodes are marked with an arrow.}
    \label{fig:prop1}
\end{figure}

\begin{proof}
This property will be proven by induction, specifying the number of nodes in the subtree as \textbf{N}. In Figure \ref{fig:prop1}, we present illustrative diagrams to exemplify the demonstration.

\begin{enumerate}
    \item \textbf{Base case N=1}. If $s$ is a leaf node, then when a successor $s'$ of $s$ is sampled it will produce a subtree with only one leaf node $s'$. See Figure \ref{fig:prop1}-(1).
    \item \textbf{Base case N=2}. If $s$ is a node whose subtree consists of only a leaf successor $s'$, then:
    \begin{enumerate}
    \item If $s'$ is selected, it is because $f(s')\leq f(s)$. In this case, a new node $s''$ will be sampled, which will become a leaf node, and $s'$ will become a non-leaf node such that $f(s')\leq f(s)$. See Figure \ref{fig:prop1}-(2a).
    \item If $s$ is selected instead, it is because $f(s') > f(s)$, and then a new leaf $s''$ will be generated, resulting in a subtree with two leaves, $s'$ and $s''$, as in (1). See Figure \ref{fig:prop1}-(2b).
    \end{enumerate}
    \item \textbf{Inductive step}. If $s$ is a node such that the $f$-value of every non-leaf node in the current subtree rooted at $s$ is bounded by $f(s)$, then:
    
    \begin{enumerate}
        \item If a leaf node $s'$ is selected, it is because $f(s')\leq f(s)$. In this case, a new node $s''$ will be sampled, which will become a leaf node, and $s'$ will become a non-leaf node such that $f(s')\leq f(s)$, as in (2a). See Figure \ref{fig:prop1}-(3a).
        \item If a non-leaf node $s'$ is selected instead, whose $f$-value already guarantees the desired property, it will only produce a new leaf node as in (2b). See Figure \ref{fig:prop1}-(3b).
    \end{enumerate}
\end{enumerate}

By induction, the $f$-value of every non-leaf node in the subtree rooted at $s$ is bounded by the current $f$-value of $s$.

\end{proof}

With this property we can guarantee an optimality bound for S-A algorithms. When a solution is found, its cost will be bounded by the current $f$-value of the initial state.

\begin{theorem}
    Let $\mathcal{P} = \langle F, X \cup U, A, s_0, G\rangle$ be a numeric planning problem with control variables that induces the transition system $\mathcal{T}(\mathcal{P})=\langle S, \U, s_0, S_G, \rightarrow\rangle$. Let $\phi$ be a sampling function such that $supp(\phi(s))=D(s)\ \forall s \in S$ and let $r_h$ be a proper rectification function, where $h$ is a heuristic function \textbf{such that } $h(s_G)=0\ \forall s_G\in S_G$, i.e., $h$ is a goal-aware heuristic. Whenever S-A$_{\phi,r_h}$ finds a solution $\pi \in \Pi_\P$, its cost is bounded by $r_h(n,s_0)$, where $n \in \mathbb{Z}^+$ is the number of times $s_0$ has been re-expanded.
\end{theorem}
\begin{proof}
    If a solution $\pi$ is found, it means that a state $s_G\in S_G$ is part of the (sub)tree rooted at $s_0$. That means that $f(s_G)\leq f(s_0)$. But $f(s_G)=g(s_G)+r_h(0,s_G)=g(s_G)+h(s_G)$, and since $s_G$ is a goal state, $h(s_G)=0$. Thus: 
    \[g(s_G)\leq f(s_0)=g(s_0)+r_h(n,s_0)=r_h(n,s_0)\]
    Since $s_0$ is the initial state, its cost is zero. With this, it has been proved that the cost of the solution found $g(s_G)$ is bounded by $r_h(n,s_0)$.
\end{proof}

The outcome of this theorem enables us to ensure a certain degree of solution quality through the thoughtful definition of the rectification functions. However, it does not offer guarantees of optimality, as the rectification functions must be defined to increase without bound in the limit. Nonetheless, in the absence of optimality guarantees, the rate of growth can be adjusted to any desired pace, thereby controlling the number of re-expansions and, consequently, the extent of exploration at each depth level before progressing to the next.


\section{Related Work}

In this section, we describe different approaches that involve or make use of continuous numeric variables or control parameters. We present the related work on resolution strategies prior to the experimentation of Section \ref{expereiments} to help understand the comparative evaluation in this section.

\paragraph{Existing planners.} As discussed in the introduction, the planners POPCORN and NextFLAP handle control parameters implicitly alongside the temporal and numeric planning constraints. These two planners address the infinite decision space caused by the continuous numeric variables through discrete forward Partial Order Planning search combined with a constraint-based formulation. Control parameters remain lifted during the search, constrained within an interval, and are periodically optimized. POPCORN inherits the use of linear programming from the COLIN planner \cite{ColesCFL2012} to optimize the control variables while NextFLAP uses Satisfiability Modulo Theories (SMT). Both linear programming and SMT solutions aim to obtain strong variable bounds, thus narrowing down the decision space delimited by the control parameters. Therefore, these planners do not treat control parameters as decision points but as constraints inherent to the problem. A former planning framework, TM-LPSAT \cite{ShinD2005}, was able to handle actions with real or interval-valued parameters among other features like exogenous events, processes, reusable metric or interval resources, in an extended version of the PDDL+ language \cite{Fox_2006}. TM-LPSTAT uses a SAT-based arithmetic constraint solver to find a solution to the system constraints and proves it is sound and complete for a subset of the aforementioned features. TM-LPSAT contributes with the theoretical grounds to demonstrate that a SAT-based planning framework can be extended to deal with problems involving continuous change to numeric quantities.

\paragraph{Hybrid systems.} Continuous numeric variables have also been used in hybrid systems that involve a mixed discrete-continuous formulation. In this setting, actions are commonly hybrid with bounded continuous control variables to represent the dynamics and evolution of the system state over time while decision variables are still discrete. In hybrid planning, the parameters that control the system are intended to represent the margins of control of the processes during execution, and they are typically interpreted as rates of change \cite{Fernandez-Gonzalez15,konming08}. This concept aligns with the \textit{processes} formalized in PDDL+ \cite{Fox_2006} and has recently been exploited to explain the behavior of hybrid systems \cite{Aineto2022ExplainingTB}. This line of work falls outside the scope of this paper, which focuses on efficiently handling continuous parameters. 

\paragraph{Rapidly-Exploring Random Trees.} 
When working with control parameters, one may consider continuous-space search algorithms, such as Rapidly-Exploring Random Trees (RRT) \cite{LaValle1998RapidlyexploringRT}, which are well-suited for navigating infinite decision spaces. Although RRT has proven effective in motion and path planning problems \cite{rrtrecentadv}, it is less suitable for structured action spaces that require applicability analysis. The application of RRT in Automated Planning is sparsely explored in the literature. To the best of our knowledge, the only framework that addresses this challenge is RRT-Plan \cite{vidal21}, though it is limited to propositional planning.  

\paragraph{Trial-Based Heuristic Tree Search.} In the literature, a reformulation of the popular UCT algorithm \cite{KocsisS06} of Monte-Carlo Tree Search (MCTS) \cite{Browne12} for classical planning problems has been explored under the framework of Trial-Based Heuristic Search (THTS) \cite{thts}. This approach uses a heuristic function that resembles the Monte-Carlo sampling rollout of the simulation phase of MCTS and allows to model any MCTS algorithm within the framework. 
A recent investigation highlights that the action selection method UCB1 of UCT is inadequate for solving classical planning problems in THTS, and addresses this limitation by proposing an alternative criterion called  UCB1-Normal2 \cite{masataro24}. We will, however, use a standard UCT with UCB1 criteria for our experimental evaluation, since the assumptions underpinning UCB1-Normal2 do not hold for infinite decision spaces, e.g., a very high likelihood relative to the variance estimator of the successors' heuristic values.

\paragraph{Search with partial expansions.} Some algorithms such as Partial Expansion A* (PEA*) \cite{pea} handle search with partial expansions for problems with large branching factors. While the branching factor in PEA* is large, it is finite, whereas we cannot generate the infinite successors of a node to determine which ones to discard. Instead, we rely on the sampling strategy to generate successors iteratively, which presents an additional challenge as we do not have access to the $f$-values of the non-generated children. The idea of Enhanced PEA \cite{epea}, which generates only nodes that satisfy a condition on their $f$-value, is an interesting direction if it could be adapted to work with domain-specific operator information in our sampling process.

\section{Experiments}
\label{expereiments}

We aim to compare our approach with existing alternative techniques that either keep decision points implicit using constraints or rely solely on local search.

\paragraph{Baselines.}
We have chosen to compare our approach with  NextFLAP planner \cite{SAPENA2024122820}, as, unfortunately, POPCORN has been confirmed to be in a non-compilable state after consultation with its authors. Internally, NextFLAP discretizes the control parameter values and performs a forward POP search. Although this results in a discretized problem which is not truly infinite, it remains the closest we can achieve using a functional planner that works for control parameters. To also compare with a search method that works for true infinite spaces, we will use a Monte-Carlo Tree Search algorithm with UCB1. Specifically, we will employ its Progressive Widening version, capable of searching in infinite decision spaces \cite{mctsPW}.

\paragraph{Domains.} 
We utilize the domains introduced in POPCORN: \cashpoint (\textcolor{Green}{$\bullet$}), \procurement (\textcolor{Mulberry}{$\bullet$}), and \terraria (\textcolor{black}{$\bullet$}), along with four domains that are extensions of domains from the latest numeric IPC\footnote{https://github.com/ipc2023-numeric}: \counters (\textcolor{red}{$\bullet$}), \blockgrouping (\textcolor{blue}{$\bullet$}), \drone (\textcolor{Aquamarine}{$\bullet$}), and \sailing (\textcolor{olive}{$\bullet$}). These extensions have been implemented  to allow for continuous increments/decrements of numeric variables and can be found in the supplementary material.

\paragraph{Algorithm instances.}
We aim to analyze multiple instances of the S-BFS algorithm to analyze the impact of different rectification functions and sampling methods. We will focus on additive rectification functions, i.e., $r_h(n,s) = h(s) + r(n)$, as they will allow us to clearly exemplify various types of growth that we are interested in analyzing: \textbf{linear} ($r_{\textup{lin}}(n)=n$), \textbf{quadratic} ($r_{\textup{qua}}(n)=n^2$) and \textbf{logarithmic} ($r_{\textup{log}}(n)=\log(1+n)$).

We will also explore several sampling strategies:

\begin{itemize}
    \item \textbf{Systematic sampling:} $\phi_s$. We sample at the extremes of the interval and then at middle points. For example, for $[0,1]$, we sample $0$, $1$, $0.5$, $0.25$, $0.75$, etc. We aim to test the impact of favoring extreme values.
    \item \textbf{Uniform sampling:} $\phi_u$. Each successor has an equal probability of being sampled. We aim to test whether randomness improves solution quality and/or coverage.
    \item \textbf{Heuristic-guided sampling:} $\phi_h$. We guide sampling using the heuristic function, making it more likely to sample states with better heuristic values. The sampling probability is given by:
    \[
    P_h(s'|s) = \frac{\left(\frac{1}{h(s') + \epsilon}\right)^\beta}{\sum_{s'' \in \textup{succ}(s)} \left(\frac{1}{h(s'') + \epsilon}\right)^\beta}, \quad s' \in \textup{succ}(s)
    \]
    where $\epsilon$ avoids division by zero and $\beta>0$. The denominator is too costly to evaluate, making direct sampling computationally impractical. We will estimate $P$ by approximating the denominator using a fixed number of uniformly sampled successors, $N$, rather than using methods like Metropolis-Hastings \cite{metropolis-hastings}, which require a huge burn-in period. We aim to test the impact of favoring heuristically promising states.
\end{itemize}

\paragraph{Setup.}
We implemented prototype versions of each domain in Python using the goal-counting heuristic \( h_{GC} \), as it is independent of the transition system\footnote{Code: https://github.com/aasomol/SBFS-ControlVariables}. Heuristics like \( h_{FF} \) \cite{Hoffmann2003TheMP} or $h_{MRP}$ \cite{hmrp}, designed for numeric planning, assume finite decision spaces. The use or adaptation of informed planning heuristics for continuous spaces is beyond the scope of this work. We defined 20 problems of various sizes for each domain. We used $\beta=1$ for heuristic-guided sampling and $\alpha=0.3$ and UCB1 for MCTS with PW. All experiments were run with a fixed seed, on a 12th Gen Intel(R) Core(TM) i9-12900KF CPU and Ubuntu 22.04 LTS, for 10 minutes and 8 GB memory limit. 

\begin{table}[t]
    \centering\small
    \begin{tabular}{|c | c|c|c|c|}
         \toprule
         \multicolumn{2}{|c|}{\textbf{S-BFS (140)}} & $\phi_s$ & $\phi_u$ & $\phi_h$ \\
         \toprule \multirow{2}{*}{$r_\textup{lin}$} & \textbf{S-G} &131 (30.16) & 126 (37.69) & 95 (33.87) \\ 
         & \textbf{S-A} & 18 (71.07) & 13 (49.42) &  11 (57.14) \\
         \midrule
         \multirow{2}{*}{$r_\textup{{qua}}$} & \textbf{S-G} & 120 (25.30) & 114 (28.19) & 91 (16.77) \\ 
         & \textbf{S-A}  & 16 (59.70) & 12 (63.92) & 9 (51.99)\\
         \midrule
         \multirow{2}{*}{$r_\textup{log}$} & \textbf{S-G} & 133 (29.13) & \textbf{136 (26.53)} & 119 (28.55) \\ 
         & \textbf{S-A} & 30 (73.90) & \textbf{31 (76.46)} &28 (67.41)\\
         \bottomrule
    \end{tabular}
    \caption{Coverage and re-expansion rate of S-G and S-A with fixed seed, aggregated for every domain.}
    \label{tab:sgsa}
\end{table}

\paragraph{Analysis of S-BFS.} 
First, we aim to analyze the behavior of different algorithm instances and determine whether a dominant strategy emerges. To this end, Table \ref{tab:sgsa} presents the performance of each specific algorithm instance by domain, showing coverage and the percentage of total re-expansions in parentheses, aggregated for all domains. Results indicate that logarithmic rectification is particularly beneficial for both algorithms, leading to a significant improvement in coverage, which is expected since optimality is not the objective. Additionally, it is observed that systematic and uniform sampling functions outperform heuristic-guided sampling. Re-expansion rate is higher in S-A algorithms due to the more informed NEC. In the supplementary material we show the performance of each configuration separated by domain. To study the behaviour of S-G and S-A in depth, we select the best configuration that minimizes the number of actions and compare performance in Figure \ref{fig:savssg}. As expected, S-A achieves higher-quality solutions than S-G but generally requires more iterations to reach them, and solves much less instances. No trend is observed for time.

\begin{figure}
    \centering
    \includegraphics[width=1\linewidth]{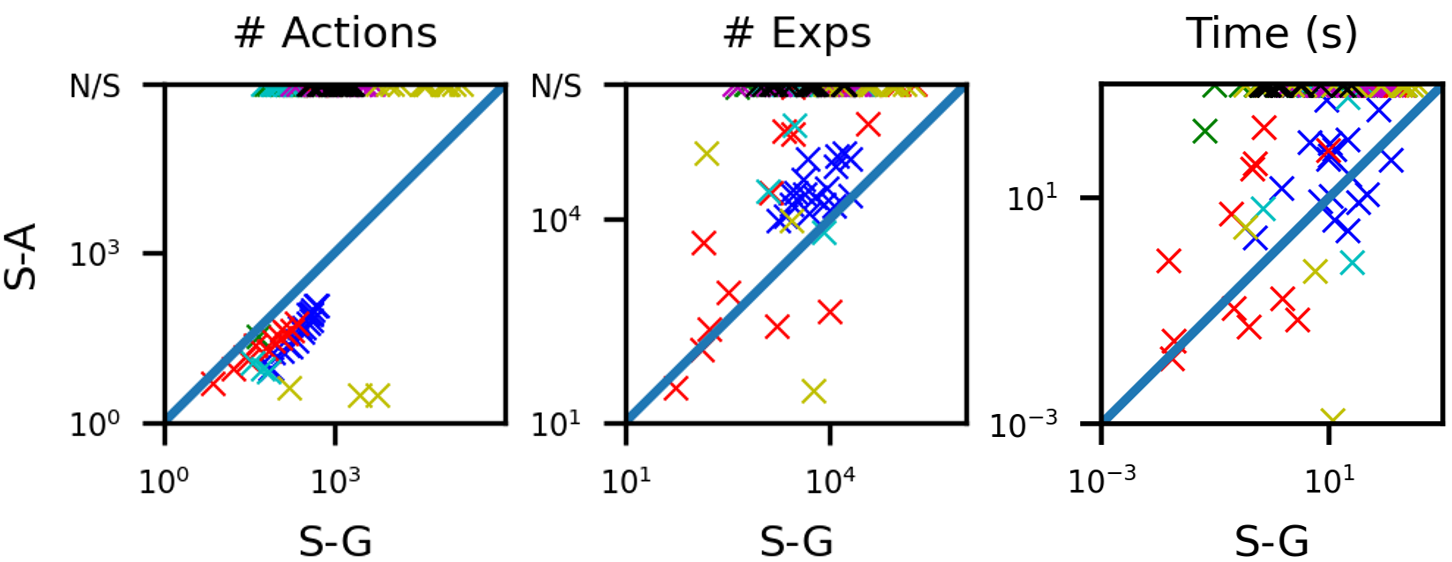}
    \caption{Comparison of S-G versus S-A in number of actions, number of expansions and time, per domain.}
    \label{fig:savssg}
\end{figure}

\paragraph{Comparison with baselines.}
We want to compare the performance of S-BFS against baselines. To do so, we first analyze the coverage of the algorithms over time. For S-G and S-A, we use the runs that minimized the number of actions. In Figure \ref{fig:survivalplot} (left), in which we depict a survival plot, i.e., coverage over time,   for every analyzed strategy, we observe that S-G successfully solves all 140 proposed problems, and S-A solves more problems than NextFLAP. Monte Carlo Tree Search solves very few problems, even when compared to NextFLAP. We will focus on NextFLAP for this reason. For the plans solved by both NextFLAP and S-BFS, it is observed in Figure \ref{fig:survivalplot} (right), in which we depict the comparison between S-BFS and NextFLAP in terms of number of actions, that the number of actions achieved by NextFLAP is lower than that of S-BFS in problems solved by both approaches. The number of iterations and the time taken by NextFLAP in comparison to S-BFS depend on the specific problem instance being solved and do not follow a clear pattern that warrants further emphasis.

\begin{figure}[t]
    \centering
    \includegraphics[width=.9\linewidth]{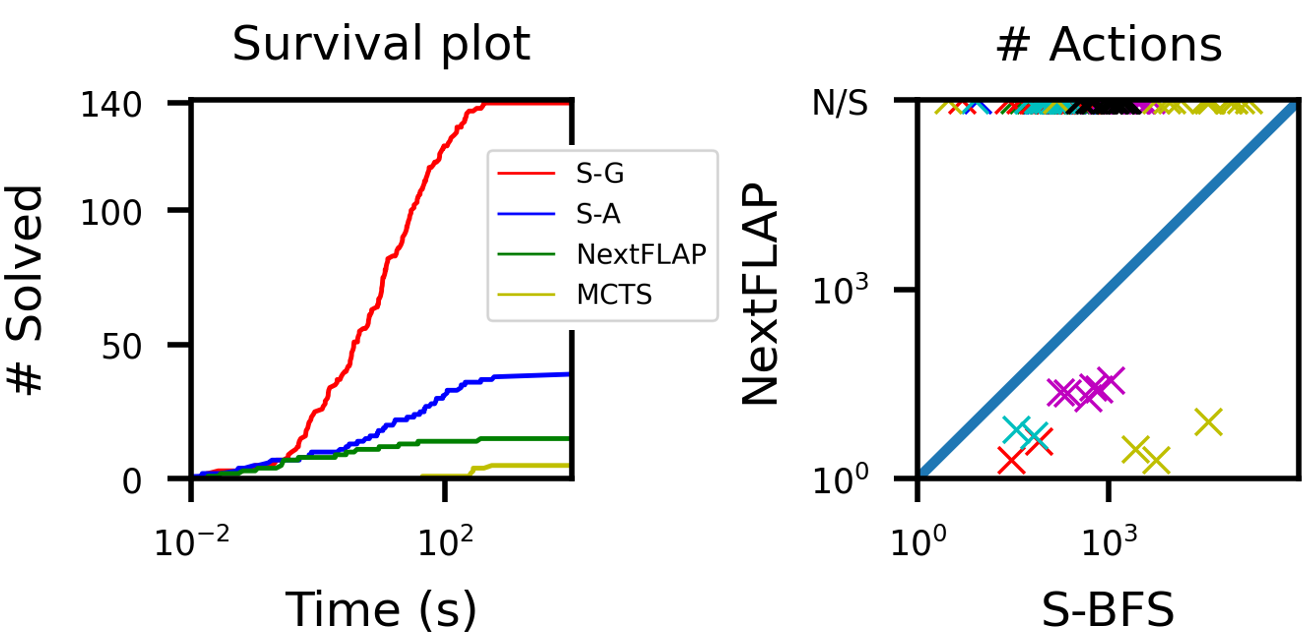}
    \caption{Survival plot of four approaches (left) and number of actions of S-BFS against NextFLAP per domain (right)}
    \label{fig:survivalplot}
\end{figure}

\paragraph{Discussion.}
In the context of the Analysis of S-BFS, we observed a clear trend toward improvement with logarithmic rectification. This suggests that the algorithm benefits from emphasizing heuristics rather than high penalties. Given that penalization is required to achieve probabilistic completeness, the use of heuristics without rectification is infeasible. Among the different forms of rectification, logarithmic growth proved to be the most effective, as it allows the heuristic to contribute most significantly to the search process.

Our results further indicate that both systematic and uniform sampling strategies yielded the best performance. Interestingly, the incorporation of the heuristic function to guide the sampling did not result in improved search efficiency. This can be attributed to the fact that the heuristic function employed exhibits numerous plateaus, causing the heuristic-guided sampling strategy to behave similarly to a uniform sampling strategy, albeit with a higher computational cost. 

On the other hand, the comparison of results with NextFLAP and MCTS demonstrates that our algorithm is able to solve a significantly higher number of instances. In the case of NextFLAP, this could be attributed to its reliance on an optimization module and the forward POP search mechanism. In contrast, for Monte-Carlo Tree Search, the advantage of our approach may stem from the inherent complexity and the progressive size scaling of the problems addressed. However, it is observed that the plans generated by NextFLAP outperform those produced by S-BFS for smaller problem instances, which could be due to the absence of optimality guarantees in our algorithm that arise from ensuring probabilistic completeness. Moreover, NextFLAP conducts a minimization of the makespan as a last planning step. However, we argue that the substantial differences in coverage compensate for this limitation.

\section{Conclusion}
In this work, we adopted a planning framework with the following characteristics: discretized time, instantaneous actions, and continuous numeric variables. We deliberately ignore durative actions in this approach because although 
we view the duration of actions as a kind of a continuous numeric parameter, we believe that the semantics associated with the use of a control parameter \textit{time} in actions would require a special formalization. 

We intend to lay the foundations for handling control parameters via search. 
For this reason, we have proposed a systematic search algorithm for solving problems with control parameters; which is based on the concept of delayed partial expansion, where a state is not fully expanded but instead incrementally expands a subset of its successors. We have demonstrated interesting properties of this algorithm that hold under certain conditions, such as probabilistic completeness and a specific notion of solution quality guarantees. Finally, we have compared our approach against other methods capable of searching in spaces defined with control parameters, showing that our algorithm outperforms existing approaches.

For future work, we aim to integrate our framework within the context of temporal planning and extend it to handle continuous-time actions as described in PDDL+.
 
Additionally, we want to study numeric planning heuristics that can take infinite decision space into account, too, for instance starting from the well studied subgoaling relaxation framework \cite{subgoaling20,kuroiwa:22:lmcut}. Overall, this work lays the foundation for a promising line of research.

\section*{Acknowledgements}

This work was partially supported by the project I+D+i AEI PID2021-127647NB-C22 funded by MICIU/AEI/10.13039/501100011033 and by FEDER/UE; and by the GENERALITAT VALENCIANA project PROMETEO CIPROM/2023/23. This work was also partially supported by the projects SERICS (PE00000014) and FAIR (B53C22003980006), both under the NRRP MUR program funded by the EU - NGEU.
Angel Aso-Mollar is partially supported by the FPU21/04273.

\bibliographystyle{named}
\bibliography{ijcai25}

\end{document}